\newif\ifJOURNAL
\JOURNALfalse
\newif\ifCONF
\CONFfalse
\newif\ifarXiv
\arXivfalse
\newif\ifWP
\WPfalse
\newif\ifFULL
\FULLfalse

\arXivtrue

\newif\ifnotCONF    
\notCONFtrue
\ifCONF\notCONFfalse\fi

\newif\ifLATIN
\ifJOURNAL
  \LATINfalse
\fi
\ifCONF
  \LATINtrue
\fi
\ifarXiv
  \LATINtrue
\fi
\ifWP
  \LATINfalse
\fi

\newif\ifnotLATIN	
\notLATINtrue
\ifLATIN\notLATINfalse\fi

\ifJOURNAL

  \newcommand{\OCMV}{vovk:2013ML}
  
  \newcommand{\OCMVII}{vovk/petej:2014UAI}

  \newcommand{\OCMXI}{OCM11}
  
\fi
\ifCONF

  \newcommand{\OCMV}{vovk:2012ACML}
  
  \newcommand{\OCMVII}{vovk/petej:2014UAI}

  \newcommand{\OCMXI}{OCM11}
  
\fi
\ifarXiv

  \newcommand{\OCMV}{vovk:arXiv1209}
  
  \newcommand{\OCMVII}{vovk/petej:arXiv1211}

  \newcommand{\OCMXI}{OCM11}
  
\fi
\ifWP

  \newcommand{\OCMV}{OCM5}
  
  \newcommand{\OCMVII}{OCM7}

  \newcommand{\OCMXI}{OCM11}
  
\fi

\newlength{\picturewidth}
\ifCONF
  \documentclass[runningheads,a4paper]{llncs}
  \usepackage{amsmath,amsfonts,amssymb,latexsym,graphicx,url,stmaryrd,algorithm,algorithmic}
  \urldef{\myemail}\path|{volodya.vovk,ivan.petej,alushaf}@gmail.com|
  
  \newcommand{\Extra}[1]{}
  \newcommand{\zzrelax}[1]{}
\fi

\ifarXiv
  \documentclass[10pt]{article}
  \usepackage{amsmath,amsthm,amsfonts,amssymb,latexsym,graphicx,url,stmaryrd,algorithm,algorithmic}
  \newcommand{\Extra}[1]{}
\fi

\ifWP
  \documentclass{article}
  \usepackage{amsmath,amsthm,amsfonts,amssymb,latexsym,graphicx,url,stmaryrd,algorithm,algorithmic}
  \input{OT2enc.def}
  
  \usepackage{CJK}
  \input{/Doc/Computing/Latex/kp.txt}
  \newcommand{\Extra}[1]{}
  \newcommand{\zzrelax}[1]{}
\fi

\ifFULL
  \usepackage{color}
  \newcommand{\Extra}[1]{\blue{#1}}
  
  \newcommand{\blue}[1]{\textcolor{blue}{#1}}
  \newcommand{\bluebegin}{\begingroup\color{blue}}
  \newcommand{\blueend}{\endgroup}

\fi

\allowdisplaybreaks[4]

\ifarXiv
  \newcommand{\zzrelax}[1]{\relax}
\fi

\newcommand{\dd}[1]{\,\textrm d{#1}}    

\DeclareMathOperator{\Prob}{\mathbb{P}}
\DeclareMathOperator{\Expect}{\mathbb{E}}

\DeclareMathOperator{\SSS}{S}   
\DeclareMathOperator{\NNN}{N}   
\let\OE\relax                   
\DeclareMathOperator{\OE}{OE}   
\DeclareMathOperator{\OF}{OF}   

\DeclareMathOperator{\CP}{CP}   

\DeclareMathOperator{\PPP}{\mathcal{P}}  
\DeclareMathOperator{\OOO}{\mathcal{O}}  
\DeclareMathOperator{\RRR}{\mathcal{R}}  

\ifCONF
  \newtheorem{hypothesis}{Hypothesis}
\fi

\ifnotCONF

  \newtheorem{theorem}{Theorem}
  
  \theoremstyle{definition}
  \newtheorem{remark}{Remark}
\fi

\title{From conformal to probabilistic prediction}

\ifarXiv
  \author{Vladimir Vovk, Ivan Petej, and Valentina Fedorova\\
  \texttt{\{volodya.vovk,ivan.petej,alushaf\}{\rm@}gmail.com}}
\fi

\ifWP
  \author{Vladimir Vovk, Ivan Petej, and Valentina Fedorova}
  
  \twodatestrue
  
\fi

\ifCONF
  
\begin{document}
  \mainmatter
  \author{Vladimir Vovk\and Ivan Petej\and Valentina Fedorova} 
  \authorrunning{Vovk, Petej, Fedorova} 
  \institute{Computer Learning Research Centre,\\
    Department of Computer Science,\\
    Royal Holloway, University of London,\\
    Egham, Surrey, UK\\[1mm]
  \myemail}
  \toctitle{From conformal to probabilistic prediction}
  \tocauthor{Vovk, Petej, Fedorova} 
  \maketitle
\fi

\ifnotCONF
  \begin{document}
  \maketitle
\fi

\begin{abstract}
  This paper proposes a new method of probabilistic prediction,
  which is based on conformal prediction.
  The method is applied to the standard USPS data set and gives encouraging results.
\end{abstract}

\section{Introduction}

In essence, conformal predictors output systems of p-values:
to each potential label of a test object a conformal predictor
assigns the corresponding p-value,
and a low p-value is interpreted as the label being unlikely.
It has been argued, especially by Bayesian statisticians,
that p-values are more difficult to interpret than probabilities;
besides, in decision problems probabilities can be easily combined with utilities
to obtain decisions that are optimal from the point of view of Bayesian decision theory.
In this paper we will apply the idea of transforming p-values into probabilities
(used in a completely different context in, e.g., \cite{vovk:1993logic}, Sect.~9,
and \cite{sellke/etal:2001})
to conformal prediction:
the p-values produced by conformal predictors will be transformed into probabilities.

The approach of this paper is as follows.
It was observed in \cite{\OCMXI} that some criteria of efficiency for conformal prediction
(called ``probabilistic criteria'')
encourage using the conditional probability $Q(y\mid x)$
as the conformity score for an observation $(x,y)$,
$Q$ being the data-generating distribution.
In this paper we extend this observation to label-conditional predictors
(Sect.~\ref{sec: criteria for p-values}).

Next we imagine that we are given a conformal predictor $\Gamma$
that is nearly optimal with respect to a probabilistic criterion
(such a conformal predictor might be an outcome of a thorough empirical study
of various conformal predictors using a probabilistic criterion of efficiency).
Essentially, this means that in the limit of a very large training set
the p-value that $\Gamma$ outputs for an observation $(x,y)$
is a monotonic transformation of the conditional probability $Q(y\mid x)$
(Theorem~\ref{thm: CP} in Sect.~\ref{sec: optimal}).

Finally, we transform the p-values back into conditional probabilities
using the distribution of p-values in the test set
(Sect.~\ref{sec: calibration}).
Following \cite{vovk:1993logic} and \cite{sellke/etal:2001},
we will say that at this step we \emph{calibrate} the p-values into probabilities,

In Sect.~\ref{sec: experiments} we give an example of a realistic situation
where use of the techniques developed in this paper
improves on a standard approach.
The performance of the probabilistic predictors considered in that section
is measured using standard loss functions, logarithmic and Brier
(Sect.~\ref{sec: criteria for probabilities}).

\subsection*{Comparisons with related work}

It should be noted that in the process of transforming p-values into probabilities
suggested in this paper
we lose a valuable feature of conformal prediction, its automatic validity.
Our hope, however, is that the advantages of conformal prediction will translate
into accurate probabilistic predictions.

There is another method of probabilistic prediction that is related to conformal prediction,
Venn prediction (see, e.g., \cite{vovk/etal:2005book}, Chap.~6, or \cite{\OCMVII}).
This method does have a guaranteed property of validity
(perhaps the simplest being Theorem~1 in \cite{\OCMVII});
however, the price to pay is that it outputs multiprobabilistic predictions
rather than sharp probabilistic predictions.
There are natural ways of transforming multiprobabilistic predictions
into sharp probabilistic predictions
(see, e.g., \cite{\OCMVII}, Sect.~4),
but such transformations, again, lead to the loss of the formal property of validity.

As preparation,
we study label-conditional conformal prediction.
For a general discussion of conditionality in conformal prediction,
see \cite{\OCMV}.
Object-conditional conformal prediction has been studied in \cite{lei/wasserman:2013}
(in the case of regression).

\section{Criteria of efficiency for label-conditional conformal predictors and transducers}
\label{sec: criteria for p-values}

Let $\mathbf{X}$ be a measurable space (the \emph{object space})
and $\mathbf{Y}$ be a finite set equipped with the discrete $\sigma$-algebra
(the \emph{label space});
the \emph{observation space} is defined to be $\mathbf{Z}:=\mathbf{X}\times\mathbf{Y}$.
A \emph{conformity measure} is a measurable function $A$ that assigns to every sequence
$(z_1,\ldots,z_n)\in\mathbf{Z}^*$ of observations
a same-length sequence $(\alpha_1,\ldots,\alpha_n)$ of real numbers
and that is equivariant with respect to permutations:
for any $n$ and any permutation $\pi$ of $\{1,\ldots,n\}$,
$$
  (\alpha_1,\ldots,\alpha_n)
  =
  A(z_1,\ldots,z_n)
  \Longrightarrow
  \left(\alpha_{\pi(1)},\ldots,\alpha_{\pi(n)}\right)
  =
  A\left(z_{\pi(1)},\ldots,z_{\pi(n)}\right).
$$
The \emph{label-conditional conformal predictor} determined by $A$ is defined by
\begin{equation}\label{eq: conformal predictor}
  \Gamma^{\epsilon}(z_1,\ldots,z_l,x)
  :=
  \left\{
    y
    \mid
    p^y>\epsilon
  \right\},
\end{equation}
where $(z_1,\ldots,z_l)\in\mathbf{Z}^*$ is a training sequence,
$x$ is a test object,
$\epsilon\in(0,1)$ is a given \emph{significance level},
and for each $y\in\mathbf{Y}$
the corresponding \emph{label-conditional p-value} $p^y$ is defined by
\begin{multline}\label{eq: p}
  p^y
  :=
  \frac
  {
    \left|\left\{i=1,\ldots,l+1\mid y_i=y\And\alpha^y_i<\alpha^y_{l+1}\right\}\right|
  }
  {
    \left|\left\{i=1,\ldots,l+1\mid y_i=y\right\}\right|
  }\\
  +
  \tau
  \frac
  {
    \left|\left\{i=1,\ldots,l+1\mid y_i=y\And\alpha^y_i=\alpha^y_{l+1}\right\}\right|
  }
  {
    \left|\left\{i=1,\ldots,l+1\mid y_i=y\right\}\right|
  },
\end{multline}
where $\tau$ is a random number distributed uniformly on the interval $[0,1]$
and the corresponding sequence of \emph{conformity scores} is defined by
\begin{equation*} 
  (\alpha_1^y,\ldots,\alpha_l^y,\alpha_{l+1}^y)
  :=
  A(z_1,\ldots,z_l,(x,y)).
\end{equation*}
It is clear that the system of \emph{prediction sets} (\ref{eq: conformal predictor})
output by a conformal predictor is nested,
namely decreasing in $\epsilon$.

The \emph{label-conditional conformal transducer} determined by $A$
outputs the system of p-values $(p^y\mid y\in\mathbf{Y})$
defined by (\ref{eq: p})
for each training sequence $(z_1,\ldots,z_l)$ of observations and each test object $x$.

\subsection*{Four criteria of efficiency} 

Suppose that, besides the training sequence, we are also given a test sequence,
and would like to measure on it
the performance of a label-conditional conformal predictor or transducer.
As usual, let us define the performance on the test set
to be the average performance (or, equivalently, the sum of performances)
on the individual test observations.
Following \cite{\OCMXI},
we will discuss the following four criteria of efficiency for individual test observations;
all the criteria will work in the same direction: the smaller the better.
\begin{itemize}
\item
  The sum $\sum_{y\in\mathbf{Y}}p^y$ of the p-values;
  referred to as the \emph{S criterion}.
  This is applicable to conformal transducers
  (i.e., the criterion is $\epsilon$-independent).
\item
  The size $\left|\Gamma^{\epsilon}\right|$ of the prediction set
  at a significance level $\epsilon$;
  this is the \emph{N criterion}.
  It is applicable to conformal predictors ($\epsilon$-dependent).
\item
  The sum of the p-values apart from that for the true label:
  the \emph{OF} (``observed fuzziness'') \emph{criterion}.
\item
  The number of false labels included in the prediction set $\Gamma^{\epsilon}$
  at a significance level $\epsilon$;
  this is the \emph{OE} (``observed excess'') \emph{criterion}.
\end{itemize}
The last two criteria are simple modifications of the first two
(leading to smoother and more expressive pictures).
\ifnotCONF\begin{remark}\label{rem: general}\fi
  Equivalently, the S criterion can be defined
  as the arithmetic mean $\frac{1}{\left|\mathbf{Y}\right|}\sum_{y\in\mathbf{Y}}p^y$ of the p-values;
  the proof of Theorem~\ref{thm: CP} below will show that, in fact,
  we can replace arithmetic mean by any mean (\cite{hardy/etal:1952}, Sect.~3.1),
  including geometric, harmonic, etc.
\ifnotCONF\end{remark}\fi

\section{Optimal idealized conformity measures for a known probability distribution}
\label{sec: optimal}

In this section we consider
the idealized case where the probability distribution $Q$
generating independent observations $z_1,z_2,\ldots$ is known
(as in \cite{\OCMXI}).
The main result of this section,
Theorem~\ref{thm: CP},
is the label-conditional counterpart of Theorem~1 in \cite{\OCMXI};
the proof of our Theorem~\ref{thm: CP}
is also modelled on the proof of Theorem~1 in \cite{\OCMXI}.
In this section we assume, for simplicity, that the set $\mathbf{Z}$ is finite
and that $Q(\{z\})>0$ for all $z\in\mathbf{Z}$.

An \emph{idealized conformity measure} is a function $A(z,Q)$
of $z\in\mathbf{Z}$ and $Q\in\PPP(\mathbf{Z})$
(where $\PPP(\mathbf{Z})$ is the set of all probability measures on $\mathbf{Z}$).
We will sometimes write the corresponding conformity scores as $A(z)$,
as $Q$ will be clear from the context.
The \emph{idealized smoothed label-conditional conformal predictor} corresponding to $A$
outputs the following prediction set $\Gamma^{\epsilon}(x)$
for each object $x\in\mathbf{X}$ and each significance level $\epsilon\in(0,1)$.
For each potential label $y\in\mathbf{Y}$ for $x$ define the corresponding \emph{label-conditional p-value} as
\begin{multline}\label{eq: p-value}
  p^y
  =
  p(x,y)
  :=
  \frac
    {Q(\{(x',y)\mid x'\in\mathbf{X} \And A((x',y),Q)<A((x,y),Q)\})}
    {Q_{\mathbf{Y}}(\{y\})}\\
  +
  \tau
  \frac
    {Q(\{(x',y)\mid x'\in\mathbf{X} \And A((x',y),Q)=A((x,y),Q)\})}
    {Q_{\mathbf{Y}}(\{y\})}
\end{multline}
(this is the idealized analogue of (\ref{eq: p})),
where $Q_{\mathbf{Y}}$ is the marginal distribution of $Q$ on $\mathbf{Y}$
and $\tau$ is a random number distributed uniformly on $[0,1]$.
The prediction set is
\begin{equation}\label{eq: prediction set}
  \Gamma^{\epsilon}(x)
  :=
  \left\{
    y\in\mathbf{Y}
    \mid
    p(x,y)>\epsilon
  \right\}.
\end{equation}
The \emph{idealized smoothed label-conditional conformal transducer} corresponding to $A$
outputs for each object $x\in\mathbf{X}$
the system of p-values $(p^y\mid y\in\mathbf{Y})$ defined by (\ref{eq: p-value});
in the idealized case we will usually use the alternative notation $p(x,y)$ for~$p^y$.

\subsection*{Four idealized criteria of efficiency}

In this subsection we will apply the four criteria of efficiency
that we discussed in the previous section
to the idealized case of infinite training and test sequences;
since the sequences are infinite, they carry all information
about the data-generating distribution $Q$.
We will write $\Gamma^{\epsilon}_A(x)$ for the $\Gamma^{\epsilon}(x)$ in (\ref{eq: prediction set})
and $p_A(x,y)$ for the $p(x,y)$ in (\ref{eq: p-value})
to indicate the dependence on the choice of the conformity measure~$A$.
Let $U$ be the uniform probability measure on the interval~$[0,1]$.

An idealized conformity measure~$A$ is:
\begin{itemize}
\item
  \emph{S-optimal} if
  $
    \Expect_{(x,\tau)\sim Q_{\mathbf{X}}\times U}
    \sum_yp_A(x,y)
    \le
    \Expect_{(x,\tau)\sim Q_{\mathbf{X}}\times U}
    \sum_yp_B(x,y)
  $
  for any idealized conformity measure $B$,
  where $Q_{\mathbf{X}}$ is the marginal distribution of $Q$ on $\mathbf{X}$;
\item
  \emph{N-optimal} if
  $
    \Expect_{(x,\tau)\sim Q_{\mathbf{X}}\times U}
    \left|\Gamma^{\epsilon}_A(x)\right|
    \le
    \Expect_{(x,\tau)\sim Q_{\mathbf{X}}\times U}
    \left|\Gamma^{\epsilon}_B(x)\right|
  $
  for any idealized conformity measure~$B$
  and any significance level~$\epsilon$;
\item
  \emph{OF-optimal} if
  \begin{equation*}
    \Expect_{((x,y),\tau)\sim Q\times U}
    \sum_{y'\ne y}p_A(x,y')
    \le
    \Expect_{((x,y),\tau)\sim Q\times U}
    \sum_{y'\ne y}p_A(x,y')
  \end{equation*}
  for any idealized conformity measure $B$;
\item
  \emph{OE-optimal} if
  \begin{equation*}
    \Expect_{((x,y),\tau)\sim Q\times U}
    \left|\Gamma^{\epsilon}_A(x)\setminus\{y\}\right|
    \le
    \Expect_{((x,y),\tau)\sim Q\times U}
    \left|\Gamma^{\epsilon}_B(x)\setminus\{y\}\right|
  \end{equation*}
  for any idealized conformity measure~$B$
  and any significance level~$\epsilon$.
\end{itemize}

The \emph{conditional probability (CP) idealized conformity measure} is
$$
  A((x,y),Q)
  :=
  Q(y\mid x).
$$
An idealized conformity measure $A$ is a (label-conditional) \emph{refinement}
of an idealized conformity measure $B$
if
\begin{equation} 
  B((x_1,y))<B((x_2,y))
  \Longrightarrow
  A((x_1,y))<A((x_2,y))
\end{equation}
for all $x_1,x_2\in\mathbf{Z}$ and all $y\in\mathbf{Y}$.
(Notice that this definition, being label-conditional,
is different from the one given in \cite{\OCMXI}.)
Let $\RRR(\CP)$ be the set of all refinements of the CP idealized conformity measure.
If $C$ is a criterion of efficiency (one of the four discussed above),
we let $\OOO(C)$ stand for the set of all $C$-optimal idealized conformity measures.

\begin{theorem}\label{thm: CP}
  $\OOO(\SSS)=\OOO(\OF)=\OOO(\NNN)=\OOO(\OE)=\RRR(\CP)$.
\end{theorem}

\begin{proof}
  We start from proving $\RRR(\CP)=\OOO(\NNN)$.
  Fix a significance level $\epsilon$.
  A smoothed confidence predictor at level $\epsilon$ is defined as a random set
  of observations $(x,y)\in\mathbf{Z}$;
  in other words, to each observation $(x,y)$ is assigned the probability $P(x,y)$
  that the observation will be outside the prediction set.
  Under the restriction that the sum of the probabilities $Q(x,y)$ of observations $(x,y)$
  outside the prediction set
  (defined as $\sum_x Q(x,y)P(x,y)$ in the smoothed case)
  is bounded by $\epsilon Q_{\mathbf{Y}}(y)$ for a fixed $y$,
  the N criterion requires us to make the sum of $Q_{\mathbf{X}}(x)$ for $(x,y)$ outside the prediction set
  (defined as $\sum_x Q_{\mathbf{X}}P(x,y)$ in the smoothed case)
  as large as possible.
  It is clear that the set should consist of the observations with the smallest $Q(y\mid x)$
  (by the usual Neyman--Pearson argument:
  cf.\ \cite{lehmann:1986}, Sect.~3.2).

  \ifFULL\bluebegin
    This argument in fact also shows that $\OOO(\NNN)\subseteq\RRR(\CP)$.
  \blueend\fi

  Next we show that $\OOO(\NNN)\subseteq\OOO(\SSS)$.
  Let an idealized conformity measure $A$ be N-optimal.
  By definition,
  \begin{equation*}
    \Expect_{x,\tau}
    \left|\Gamma^\epsilon_A(x)\right|
    \le
    \Expect_{x,\tau}
    \left|\Gamma^\epsilon_B(x)\right|
  \end{equation*}
  for any idealized conformity measure $B$ and any significance level $\epsilon$.
  Integrating over $\epsilon\in(0,1)$ and swapping the order of integrals and expectations,
  \begin{equation}\label{eq: N-S}
    \Expect_{x,\tau} \int_0^1 \left|\Gamma^\epsilon_A(x)\right| \dd{\epsilon}
    \le
    \Expect_{x,\tau} \int_0^1 \left|\Gamma^\epsilon_B(x)\right| \dd{\epsilon}.
  \end{equation}
  Since
  $$
    \left|\Gamma^\epsilon(x)\right|
    =
    \sum_{y\in\mathbf{Y}}1_{\{p(x,y) > \epsilon\}},
  $$
  we can rewrite \eqref{eq: N-S},
  after swapping the order of summation and integration, as
  \begin{equation*}
    \Expect_{x,\tau}
    \sum_{y \in \mathbf{Y}}
    \left(
      \int_0^1
        1_{\{p_A(x,y) > \epsilon\}}
      \dd{\epsilon}
    \right)
    \le
    \Expect_{x,\tau}
    \sum_{y \in \mathbf{Y}}
    \left(
      \int_0^1 1_{\{p_B(x,y) > \epsilon\}} \dd{\epsilon}
    \right).
  \end{equation*}
  Since
  $$
    \int_0^1
      1_{\{p(x,y) > \epsilon\}}
    \dd{\epsilon}
    =
    p(x,y),
  $$
  we finally obtain
  \begin{equation*}
    \Expect_{x,\tau}
    \sum_{y \in \mathbf{Y}}
    p_A(x,y)
    \le
    \Expect_{x,\tau}
    \sum_{y \in \mathbf{Y}}
    p_B(x,y).
  \end{equation*}
  Since this holds for any idealized conformity measure $B$,
  $A$ is S-optimal.

  The argument in the previous paragraph in fact shows that
  $\OOO(\SSS)=\OOO(\NNN)=\RRR(\CP)$.
  \ifnotCONF
    Indeed, that argument shows that
    \begin{equation*}
      \sum_{y \in \mathbf{Y}}
      p(x,y)
      =
      \int_0^1
      \left|\Gamma^\epsilon(x)\right|
      \dd{\epsilon},
    \end{equation*}
    and so to optimize a conformity measure in the sense of the S criterion
    it suffices to optimize it in the sense of the N criterion
    for all $\epsilon$ simultaneously
    (which can, and therefore should, be done).
    More generally, for any continuous increasing function $\phi$ we have
    \begin{multline*}
      \sum_{y \in \mathbf{Y}}
      \phi(p(x,y))
      =
      \sum_{y \in \mathbf{Y}}
      \int_0^1
      1_{\{\phi(p(x,y))>\epsilon\}}
      \dd{\epsilon}
      =
      \int_0^1
      \sum_{y \in \mathbf{Y}}
      1_{\{p(x,y)>\phi^{-1}(\epsilon)\}}
      \dd{\epsilon}\\
      =
      \int_0^1
      \left|\Gamma^{\phi^{-1}(\epsilon)}(x)\right|
      \dd{\epsilon}
      =
      \int
      \left|\Gamma^{\epsilon'}(x)\right|
      \phi'(\epsilon')
      \dd{\epsilon'},
    \end{multline*}
    which proves Remark~\ref{rem: general}.
  \fi

  The equality $\OOO(\SSS)=\OOO(\OF)$ follows from
  $$
    \Expect_{x,\tau} \sum_{y} p(x,y)
    =
    \Expect_{(x,y),\tau} \sum_{y'\ne y}p(x,y') + \frac12,
  $$
  where we have used the fact that $p(x,y)$ is distributed uniformly on $[0,1]$
  when $((x,y),\tau)\sim Q\times U$
  (see \cite{vovk/etal:2005book} and \cite{\OCMXI}).

  Finally, we notice that $\OOO(\NNN)=\OOO(\OE)$.
  Indeed, for any significance level $\epsilon$,
  $$
    \Expect_{x,\tau} |\Gamma^\epsilon(x)|
    =
    \Expect_{(x,y),\tau} |\Gamma^\epsilon(x) \setminus \{y\}|
    +
    (1-\epsilon),
  $$
  again using the fact that $p(x,y)$ is distributed uniformly on $[0,1]$
  and so $\Prob_{(x,y),\tau}(y\in\Gamma^\epsilon(x)) = 1 - \epsilon$.
\end{proof}

\section{Criteria of efficiency for probabilistic predictors}
\label{sec: criteria for probabilities}

Given a training set $(z_1,\ldots,z_l)$ and a test object $x$,
a probabilistic predictor outputs a probability measure $P\in\PPP(\mathbf{Y})$,
which is interpreted as its probabilistic prediction for the label $y$ of $x$;
we let $\PPP(\mathbf{Y})$ stand for the set of all probability measures on $\mathbf{Y}$.
The two standard way of measuring the performance of $P$ on the actual label $y$
are the \emph{logarithmic} (or \emph{log}) \emph{loss} $-\ln P(\{y\})$
and the \emph{Brier loss}
$$
  \sum_{y'\in\mathbf{Y}}
  \Bigl(1_{\{y'=y\}}-P(\{y'\})\Bigr)^2,
$$
where $1_E$ stands for the indicator of an event $E$:
$1_E=0$ if $E$ happens and $1_E=0$ otherwise.
The efficiency of probabilistic predictors will be measured
by these two loss functions.
\ifFULL\bluebegin
  Remember that probabilistic predictors do not posses any properties
  of automatic validity
  (unlike Venn predictors, which are, however, multiprobabilistic predictors).
\blueend\fi

\ifFULL\bluebegin
  These are proper loss functions.
\blueend\fi

Suppose we have a test sequence
$(z_{l+1},\ldots,z_{l+k})$, where $z_i=(x_i,y_i)$ for $i=l+1,\ldots,l+k$,
and we want to evaluate the performance of a probabilistic predictor
(trained on a training sequence $z_1,\ldots,z_l$) on it.
In the next section we will use
the \emph{average log loss}
$$
  -\frac1k
  \sum_{i=l+1}^{l+k}
  \ln P_i(\{y_i\})
$$
and the \emph{standardized Brier loss}
$$
  \sqrt
  {
    \frac{1}{k\left|\mathbf{Y}\right|}
    \sum_{i=l+1}^{l+k}
    \sum_{y'\in\mathbf{Y}}
    \Bigl(1_{\{y'=y_i\}}-P_i(\{y'\})\Big)^2
  },
$$
where $P_i\in\PPP(\mathbf{Y})$ is the probabilistic prediction for $x_i$.
Notice that in the binary case, $\left|\mathbf{Y}\right|=2$,
the average log loss coincides with the mean log error
(used in, e.g., \cite{\OCMVII}, (12))
and the standardized Brier loss
coincides with the root mean square error
(used in, e.g., \cite{\OCMVII}, (13)).

\section{Calibration of p-values into conditional probabilities}
\label{sec: calibration}

\ifFULL\bluebegin
  We can use a hold-out set for calibration
  (say nonparametric, using monotonic regression, as in \cite{zadrozny/elkan:2002} in a related context).
  This might be too wasteful, but still we should run experiments.
  In this section we will discuss an alternative approach:
  how to calibrate p-values using the test set.
\blueend\fi

The argument of this section will be somewhat heuristic,
and we will not try to formalize it in this paper.
Fix $y\in\mathbf{Y}$.
Suppose that $q:=P(y\mid x)$ has an absolutely continuous distribution with density $f$
when $x\sim Q_{\mathbf{X}}$.
(In other words, $f$ is the density of the image of $Q_{\mathbf{X}}$ under the mapping
$x\mapsto P(y\mid x)$.\ifFULL\bluebegin\
  This assumption contradicts that assumption made earlier that $\mathbf{Z}$ is finite.\blueend\fi)
For the CP idealized conformity measure,
we can rewrite (\ref{eq: p-value}) as
\begin{equation}\label{eq: ideal p-value}
  p(q)
  :=
  \left.
    \int_0^q
    q'
    f(q')
    dq'
  \middle/
    D
  \right.,
\end{equation}
where $D:=Q_{\mathbf{Y}}(\{y\})$;
alternatively, we can set
$
  D
  :=
  \int_0^1
  q'
  f(q')
  dq'
$
to the normalizing constant ensuring that $p(1)=1$.
To see how \eqref{eq: ideal p-value} is a special case of \eqref{eq: p-value}
for the CP idealized conformity measure,
notice that the probability that $Y=y$ and $P(Y\mid X)\in(q',q'+dq')$, where $(X,Y)\sim f$,
is $q'f(q')dq'$.
In (\ref{eq: ideal p-value}) we write $p(q)$ rather than $p^y$ since $p^y$ depends on $y$
only via $q$.

\begin{algorithm}[bt]
  \caption{Conformal-type probabilistic predictor}
  \label{alg: PP}
  \begin{algorithmic}
    \renewcommand{\algorithmicrequire}{\textbf{Input:}}
    \renewcommand{\algorithmicensure}{\textbf{Output:}}
    \REQUIRE training sequence $(z_1,\ldots,z_l)\in\mathbf{Z}^l$
    \REQUIRE calibration sequence $(x_{l+1},\ldots,x_{l+k})\in\mathbf{X}^k$
    \REQUIRE test object $x_0$
    \ENSURE probabilistic prediction $P\in\PPP(\mathbf{Y})$ for the label of $x_0$
    \FOR{$y\in\mathbf{Y}$}
      \STATE for each $x_i$ in the calibration sequence find the p-value $p_i^y$
        by~\eqref{eq: p}
      \STATE \qquad (with $l+i$ in place of $l+1$)
      \STATE let $g_y$ be the antitonic density on $[0,1]$
        fitted to $p_{l+1}^y,\ldots,p_{l+k}^y$
      \STATE find the p-value $p_0^y$
        by~\eqref{eq: p} (with $0$ in place of $l+1$)
      \STATE for each $y\in\mathbf{Y}$, set $P'(\{y\}):=g_y(1)/g_y(p_0^y)$
    \ENDFOR
    \STATE set $P(\{y\}):=P'(\{y\})/\sum_{y'}P'(\{y'\})$ for each $y\in\mathbf{Y}$
  \end{algorithmic}
\end{algorithm}

We are more interested in the inverse function $q(p)$, which is defined by the condition
$$
  p
  =
  \left.
    \int_0^{q(p)}
    q'
    f(q')
    dq'
  \middle/
    D
  \right..
$$
When $q\sim f$, we have
$$
  \Prob(p(q)\le a)
  =
  \Prob(q\le q(a))
  =
  \int_0^{q(a)} f(q') dq'.
$$
Therefore, when $q\sim f$, we have
$$
  \Prob(a\le p(q)\le a+da)
  =
  \int_{q(a)}^{q(a+da)} f(q') dq'
  \approx
  \frac{1}{q(a)}
  \int_{q(a)}^{q(a+da)} q' f(q') dq'
  =
  \frac{Dda}{q(a)},
$$
and so
$$
  q(c)
  \approx
  \left.
    D
  \middle/\;
    \frac{\Prob(c\le p(q)\le c+dc)}{dc}.
  \right.
$$

This gives rise to the algorithm given as Algorithm~\ref{alg: PP},
which uses real p-values \eqref{eq: p} instead of the ideal p-values \eqref{eq: p-value}.
The algorithm is transductive in that it uses a training sequence of labelled observations
and a calibration sequence of unlabelled objects
(in the next section we use the test sequence as the calibration sequence);
the latter is used for calibrating p-values into conditional probabilities.
Given all the p-values for the calibration sequence with postulated label $y$,
find the corresponding antitonic density $g(p)$
(remember that the function $q(p)$ is known to be monotonic, namely isotonic)
using Grenander's estimator
(see \cite{grenander:1956} or, e.g., \cite{devroye:1987}, Chap.~8).
Use $D/g(p)$ as the calibration function,
where $D:=g(1)$ is chosen in such a way that a p-value of $1$ is calibrated
into a conditional probability of $1$.
(Alternatively, we could set $D$ to the fraction of observations labelled as $y$
in the training sequence;
this approximates setting $D:=Q_{\mathbf{Y}}(\{y\})$.)
The probabilities produced by this procedure are not guaranteed
to lead to a probability measure: the sum over $y$ can be different from 1
(and this phenomenon has been observed in our experiments).
Therefore, in the last line of Algorithm~\ref{alg: PP}
we normalize the calibrated p-values to obtain genuine probabilities.

\ifFULL\bluebegin
  Grenander's estimator achieves MISE rate of convergence $n^{-1/3}$
  (without any extra conditions):
  is this true?
  (It is true for $L_1$: see, e.g., \cite{devroye/gyorfi:1985}, Sect.~7.7.)
  If the density $g$ is assumed to have $m$ derivatives,
  the optimal rate of convergence without order conditions is $n^{-2m/(2m+1)}$,
  and further assuming antitonicity does not improve it
  (see \cite{efromovich:2001} and references therein).
  Kiefer \cite{kiefer:1982}: the rate of convergence is not affected when $m=1$.
  Efromovich \cite{efromovich:2001}:
  even the constant is not affected, for any $m\in\{1,2,\ldots\}$.

  \begin{remark}
    The topic of this paper is how to transform conformal predictors
    into probabilistic predictors.
    Moving in the opposite direction,
    from probabilistic to conformal predictors,
    seems to be much easier:
    given a probabilistic predictor,
    a natural conformity measure $\alpha_i$ for an observation $z_i=(x_i,y_i)$
    in a sequence $z_1,\ldots,z_n$ is the probability $\alpha_i:=P(y_i)$,
    where $P$ is the probabilistic prediction for the label of $x_i$
    found using that probabilistic predictor from $z_1,\ldots,z_n$
    (or $z_1,\ldots,z_{i-1},z_{i+1},\ldots,z_n$) as the training set.
  \end{remark}
\blueend\fi

\section{Experiments}
\label{sec: experiments}

In our experiments we use the standard USPS data set of hand-written digits.
The size of the training set is 7291, and the size of the test set is 2007;
however, instead of using the original split of the data into the two parts,
we randomly split all available data (the union of the original training and test sets)
into a training set of size 7291 and test set of size 2007.
(Therefore, our results somewhat depend on the seed used by the random number generator,
but the dependence is minor and does not affect our conclusions at all;
we always report results for seed 0.)

A powerful algorithm for the USPS data set
is the 1-Nearest Neighbour (1-NN) algorithm using tangent distance
\cite{simard/etal:1993}.
However, it is not obvious how this algorithm could be transformed
into a probabilistic predictor.
On the other hand, there is a very natural and standard way of extracting probabilities
from support vector machines,
which we will refer to it as \emph{Platt's algorithm} in this paper:
it is the combination of the method proposed by Platt \cite{platt:2000}
with pairwise coupling \cite{wu:2004}
(unlike our algorithm, which is applicable to multi-class problems directly,
Platt's method is directly applicable only to binary problems).
In this section we will apply our method to the 1-NN algorithm with tangent distance
and compare the results to Platt's algorithm
as implemented in the function \texttt{svm} from the \texttt{e1071} R package
(for our multi-class problem this function calculates probabilities
using the combination of Platt's binary method and pairwise coupling).

There is a standard way of turning a distance
into a conformal predictor (\cite{vovk/etal:2005book}, Sect.~3.1):
namely, the conformity score $\alpha_i$ of the $i$th observation
in a sequence of observations can be defined as
\begin{equation}\label{eq: NN}
  \frac
  {
    \min_{j: y_j \ne y_i}
    d(x_i,x_j)
  }
  {
    \min_{j\ne i: y_j = y_i}
    d(x_i,x_j)
  },
\end{equation}
where $d$ is the distance;
the intuition is that an object is considered conforming
if it is close to an object labelled in the same way
and far from any object labelled in a different way.

\begin{table}[tb]
\caption{The performance of the two algorithms, Platt's
  (with the optimal values of parameters)
  and the conformal-type probabilistic predictor
  based on 1-Nearest Neighbour with tangent distance}
\label{tab: performance}
\begin{center}
\begin{tabular}{r|c|c}
  \hline
  algorithm & average log loss & standardized Brier loss\\
  \hline\hline
  optimized Platt & 0.06431 & 0.05089\\
  conformal-type 1-NN & 0.04958 & 0.04359\\
\end{tabular}
\end{center}
\end{table}

\begin{table}[tb]
\caption{The performance of Platt's algorithm with the polynomial kernels of various degrees
  for the cost parameter $C=10$}
\label{tab: Platt}
\begin{center}
\begin{tabular}{r|c|c}
  \hline
  degree & average log loss & standardized Brier loss\\
  \hline\hline
  1 & 0.12681 & 0.07342\\
  2 & 0.09967 & 0.06109\\
  3 & 0.06855 & 0.05237\\
  4 & 0.11041 & 0.06227\\
  5 & 0.09794 & 0.06040
\end{tabular}
\end{center}
\end{table}

Table~\ref{tab: performance} compares
the performance of the conformal-type probabilistic predictor
based on the 1-NN conformity measure \eqref{eq: NN},
where $d$ is tangent distance,
with the performance of Platt's algorithm
with the optimal values of its parameters.
The conformal predictor is parameter-free
but Platt's algorithm depends on the choice of the kernel.
We chose the polynomial kernel of degree~3
(since it is known to produce the best results: see \cite{vapnik:1998}, Sect.~12.2)
and the cost parameter $C:=2.9$ in the case of the average log loss
and $C:=3.4$ in the case of the standardized Brier loss
(the optimal values in our experiments).
(Reporting the performance of Platt's algorithm with optimal parameter values
may look like data snooping, but it is fine in this context
since we are helping our competitor.)
Table~\ref{tab: Platt} reports the performance of Platt's algorithm as function
of the degree of the polynomial kernel with the cost parameter set at $C:=10$
(the dependence on $C$ is relatively mild,
and $C=10$ gives good performance for all degrees that we consider).

\ifFULL\bluebegin
  We do the usual normalization;
  it is interesting that Khutsishvili's method (\cite{vovk/zhdanov:2009}, appendix),
  which work extremely well for normalizing bets,
  leads to much poorer predictions
  (as perhaps was to be expected: Khutsishvili's theory is specifically designed
  for extracting probabilities from bets).
\blueend\fi

\ifFULL\bluebegin
  \section{Conclusion}

  This paper has proposed a way to turn conformal predictors into probabilistic ones.
  But perhaps it is not very efficient,
  and it appears that the main source of inefficiency
  is a separate treatment of different classes at the stage of calibrating p-values.
\blueend\fi

\subsubsection*{Acknowledgments.}

In our experiments we used the R package
\texttt{e1071}
(by David Meyer, Evgenia Dimitriadou, Kurt Hornik, Andreas Weingessel, Friedrich Leisch,
Chih-Chung Chang, and Chih-Chen Lin)
and the implementation of tangent distance by Daniel Keysers.
This work was partially supported by EPSRC (grant EP/K033344/1, first author)
and Royal Holloway, University of London (third author).

\end{document}